\def\year{2018}\relax
\documentclass[letterpaper]{article} 
\usepackage{aaai18}  
\usepackage{times}  
\usepackage{helvet}  
\usepackage{courier}  
\usepackage{url}  
\usepackage{graphicx}  
\usepackage{graphicx,amssymb,amsmath}
\usepackage{subfigure}

\usepackage{amsmath, amsfonts, amssymb, xspace, color}
\usepackage{amsthm}
\newtheorem{definition}{Definition}
\newcommand{\vpara}[1]{\noindent\textbf{#1}}
\newcommand{\etal}{\emph{et~al.}\xspace} 
\newcommand{\ie}{\emph{i.e.}\xspace} 
\newcommand{\eg}{\emph{e.g.}\xspace} 
\newcommand{\our}{\textsc{SHMM}\xspace}

\newtheorem{theorem}{Theorem}
\newtheorem{lemma}{Lemma}

\begin{document}
	%
	\title{A Spherical Hidden Markov Model for Semantics-Rich Human Mobility Modeling}
	
	\author{
		Wanzheng Zhu\thanks{Equal contribution.},  Chao Zhang\footnotemark[1], Shuochao Yao, Xiaobin Gao, Jiawei Han\\
		University of Illinois at Urbana-Champaign, Urbana, IL, USA\\
		\{wz6, czhang82, syao9, xgao16, hanj\}@illinois.edu
	}

	\maketitle

	\begin{abstract}
		
		
		We study the problem of modeling human mobility from semantic trace data,
		wherein each GPS record in a trace is associated with a text message that
		describes the user's activity.  Existing methods fall short in unveiling human
		movement regularities for such data, because they either do not model the text
		data at all or suffer from text sparsity severely. We propose \our, a
		multi-modal spherical hidden Markov model for semantics-rich human mobility
		modeling. Under the hidden Markov assumption, \our models the generation
		process of a given trace by jointly considering the observed location, time,
		and text at each step of the trace. The distinguishing characteristic of \our
		is the text modeling part. We use fixed-size vector representations to encode
		the semantics of the text messages, and model the generation of the
		$l_2$-normalized text embeddings on a unit sphere with the von Mises-Fisher
		(vMF) distribution. Compared with other alternatives like multi-variate
		Gaussian,  our choice of the vMF distribution not only incurs much fewer
		parameters, but also better leverages the discriminative power of text
		embeddings in a directional metric space. The parameter inference for the vMF
		distribution is non-trivial since it involves functional inversion of ratios of
		Bessel functions.  We theoretically prove, for the first time, that: 1) the
		classical Expectation-Maximization algorithm is able to work with vMF
		distributions; and 2) while closed-form solutions are hard to be obtained for
		the M-step, Newton's method is guaranteed to converge to the optimal solution
		with quadratic convergence rate. We have performed extensive experiments on
		both synthetic  and real-life data. The results on synthetic data verify our
		theoretical analysis; while the results on real-life data demonstrate that \our
		learns meaningful semantics-rich mobility models, outperforms state-of-the-art
		mobility models for next location prediction, and incurs lower training cost. The code and datasets are available at \url{https://github.com/WanzhengZhu/SHMM}.
	\end{abstract}

	\section{Introduction}
	
	Uncovering human mobility patterns is not only a fundamental task for human
	behavioral analysis, but also an important building block for urban planning,
	traffic forecasting, mobile health applications, and location-based recommender
	systems \cite{gonzalez2008understanding,kitamura2000micro,zhang2016gmove}.
	Recent years are witnessing an increasing importance of modeling human mobility
	from semantic trace data, where each record in a trace is associated with a
	text message that describes the user's activity. With the wide proliferation of
	mobile devices and the ubiquitous access to the mobile Internet, massive
	semantic trace data are being collected by various social media services (\eg,
	Twitter, Instagram, Facebook) and phone carriers on a daily basis
	\cite{ChengCLS11,wu2015semantic,ZhangZYZZKWH16,zhang2017regions}. Meanwhile,
	raw GPS trajectories can be readily linked with external sources (\eg, map
	data, land uses) to annotate each record with rich semantic information
	\cite{wu2015semantic}.

	The wide availability of semantic trace data necessitates a shift in the
	paradigm of human mobility modeling --- it becomes possible to interpret human
	mobilities in a semantics-rich way. In addition to uncovering the
	spatiotemporal patterns of human movements, we could move one step further to
	understand \emph{what} are people's activities at different regions, and
	\emph{how} and \emph{why} people move from one region to another. Such
	semantics-rich knowledge not only enables us to understand human mobility in a
	more interpretable way, but also plays an important role for prediction and
	decision making in various downstream applications.

	Unfortunately, learning semantics-rich human mobility models from semantic
	trace data is a challenging problem that remains largely unsolved by existing
	techniques.  Traditional mobility modeling techniques
	\cite{giannotti2007trajectory,LiDHKN10,Cho2011,MathewRM12} mostly focus on
	mining pure spatiotemporal regularities and cannot handle the text information
	in semantic traces. Recently, there have been research efforts that attempt to
	integrate the text information into the mobility modeling process based on the
	bag-of-words model
	\cite{Ying2011,wu2015semantic,zhang2016gmove}. Nevertheless,
	these methods are unable to make the best use of the text information. First,
	by considering each keyword as an independent dimension, they do not model the
	correlations between keywords (\eg, `car', `taxi' and `drive') and may fail to
	correlate semantically similar messages.  Further, since the vocabulary size is
	often large, their performance is limited by the high dimensionality and text
	sparsity, and meanwhile results in high computational overhead in the model
	learning process.
	
	

	To learn semantics-rich mobility models from semantic trace data, we propose
	\our, a method that uncovers human mobility regularities by statistically
	modeling the generation process of the given trace data. \our is a multi-modal
	spherical hidden Markov model (HMM). Under the hidden Markov assumption, it
	jointly models the generation of the observed location, time, and text at each
	step of an input trace. While the low-dimensional location and time can be
	modeled with Gaussian distributions, the key challenge is to capture the
	semantics of the high-dimensional text messages and model textual semantics.
	To address this challenge, we use fixed-size vector representations to encode
	the semantics of the text messages, which has been recently shown successful
	for a wide variety of NLP tasks
	\cite{mikolov2013efficient,mikolov2013distributed}. With the derived text
	embeddings, we further model the generation of the $l_2$-normalized text
	embeddings on a unit sphere with the von Mises-Fisher (vMF) distribution
	\cite{fisher1953dispersion}.  Compared with other alternatives like
	multi-variate Gaussian, our choice of the vMF distribution not only incurs much
	fewer parameters, but also better unleashes the discriminative power of text
	embeddings in a directional metric space.

	In the parameter inference process of our \our model, we use the classical
	Expectation-Maximization (EM) algorithm \cite{dempster1977maximum}.  However,
	since the vMF distribution has a complicated mathematical form, literature so
	far has not yet proved that EM algorithm is able to work with the vMF distribution.
	We, for the first time, theoretically prove the feasibility of using the EM
	algorithm on the vMF distribution.  Furthermore, while closed-form solutions
	are hard to be obtained for the M-step, we prove that using Newton's method is
	guaranteed to converge to the optimal solution with quadratic convergence rate. 
	
	

	To summarize, we make the following contributions:


	
	\begin{enumerate}
		
		\item We propose a spherical hidden Markov model for human mobility
		modeling with semantic trace data. Compared with existing mobility models,
		our method is novel in that it uses embeddings to encode the semantics of
		text messages and the von-Mises Fisher distribution to model the generation 
		of text embeddings.

		\item We provide rigorous theoretical proof to show that
		the EM algorithm is able to work with the vMF distribution. Also, while
		obtaining closed-form solutions for the M-step is intractable, we prove that Newton's method is guaranteed to converge to the optimal
		solution with quadratic convergence rate.  Some other properties of the vMF distribution and
		modified Bessel functions are also studied.  
		
		\item We perform extensive experiments on both synthetic and real-life data.
		The results on synthetic data verify our theoretical analysis; while the
		results on real-life data demonstrate that \our learns meaningful
		semantics-rich mobility models, outperforms state-of-the-art mobility
		models for next location prediction, and incurs lower training cost.
		
	\end{enumerate}

	\section{Problem Description}
	
	
	
	We study the problem of modeling human mobility from semantic trace data.
	Semantic traces are text-rich GPS traces where each GPS record is associated
	with a text message that describes the user's activity.  We provide the formal
	definition of a semantic trace as follows.
	
	\begin{definition} [Semantic Trace]
		The semantic trace of a user $u$ is a time-ordered sequence $S = [x_1,
		x_2, \ldots, x_R]$. Each record $x_i$ is a tuple $(t_i, l_i, m_i)$ where: (1)
		$t_i$ is the timestamp scalar; (2) $l_i$ is a two-dimensional vector representing
		the location of the user at time $t_i$; and (3) $m_i$ is a text message vector
		describing the activity of the user.
	\end{definition}
	
	
	To capture the semantics of user activities, we use distributed representations
	for the text messages in our model. Specifically, we first use the CBOW model
	\cite{mikolov2013distributed} to obtain fixed-size vector representations (\ie,
	embeddings) for the keywords in the given corpus. The parameters used are:
	\textit{min-count}=10, \textit{size}=30, \textit{window}=5,
	\textit{sample}=$10^{-4}$, \textit{negative}=5. As each text message usually
	consists of multiple keywords, we compute the TF-IDF weights of the keywords
	and use the weighted average of keyword embeddings to derive the embedding of
	the message $m_i$ \cite{le2014distributed,arora2016simple}.

	Now we are ready to formulate our mobility modeling problem.  Given the
	semantic traces for a group of users $D = \{S_1, S_2, \ldots, S_G\}$, our work
	aims to build semantics-rich mobility models for those users. The result
	mobility model is expected to address two aspects regarding human mobility: (1)
	\textbf{Discovering latent states.} The first aspect is to discover the latent
	states that govern people's activities. A latent state is an abstraction of
	\emph{what} the user is doing around \emph{where} during \emph{when}.  Examples
	include \emph{shopping} in the \emph{5th Ave} at \emph{5pm}, and \emph{watching
		a film} at the \emph{AMC theater} in the \emph{evening}.  (2)
	\textbf{Characterizing transition regularity.} The second aspect is to
	characterize how users move sequentially between the latent states.  For
	example, after shopping in the 5th Ave, what activities will the users do next?
	We aim to characterize people's transitions among the latent states in a
	concise and probabilistic way.

	
	


	\section{The \our Model}
	
	In this section, we describe SHMM in detail and describe the
	parameter inference procedure. 
	
	\subsection{Model Description}
	
	Consider a sequence of chronologically ordered records $x_1, x_2, \ldots, x_R$
	of a user $u$. In \our, we adopt the hidden Markov assumption, \ie, assuming
	each record $x_i$ is generated from a latent state $z_i$, and the sequence
	$z_1, z_2, \ldots, z_R$ follows a Markov process. The Markov process is
	parameterized by an  initial probability matrix $\pi$ over the latent states,
	as well as a matrix $A$ that specifies the transition probabilities among the
	latent states.  When generating $x_i$ from $z_i$, we assume the  location
	$l_i$,  the timestamp $t_i$, and the text embedding $m_i$ are generated
	independently. Therefore, the conditional probability $p(x_i|z_i )$ is given by
	$p(x_i|z_i ) = p(t_i|z_i) \cdot p(l_i|z_i) \cdot p(m_i|z_i)$.
	
	For each record $x_i$, we assume the following distributions for each
	component: 1) the timestamp $t_i$ is generated from a univariate Gaussian with
	mean $\mu_t$ and variance $\sigma_t$, \ie $p(t_i|z_i) = \mathcal{N} (t_i |
	\mu_t, \sigma_t)$, where $t_i$ indicates the time in a day; 2) the location
	$l_i$ is generated from a bivariate Gaussian with mean $\mu_l$ and covariance
	matrix $\Sigma_l$, \ie  $p(l_i|z_i) = \mathcal{N} (l_i | \mu_l, \Sigma_l)$; 3)
	the message vector $m_i$ is generated from the von Mises-Fisher (vMF)
	distribution with mean direction $\mu$ and concentration parameter $\kappa$,
	\ie $p(m_i|z_i) = \text{vMF}(m_i | \mu, \kappa)$. 
	
	
	While the Gaussian distribution is suitable for modeling timestamps and
	locations, it is problematic for modeling text embeddings. The reason is
	two-fold. First, using Gaussian distributions to model text embeddings would 
	lead to a large
	co-variance matrix with too many parameters. Second, previous research has
	demonstrated that the cosine distance better reflects the semantic similarity
	between text embeddings compared to the Euclidean distance, \ie, the
	discriminative power of the text embeddings is stronger in a directional metric
	space.  Our choice of the vMF distribution addresses the above two issues.  A
	vMF distribution is defined on the $p$-dimensional unit sphere, parameterized
	by two parameters: a mean direction $\mu$ and a concentration parameter
	$\kappa$. The mean $\mu$ specifies the direction of the semantic focus of the
	text embeddings, while $\kappa$ controls how concentrated the text
	embeddings are around the mean direction.
	The larger $\kappa$ is, the more concentrated the text embeddings are around
	the mean direction. 
	Formally, the probability density function of a vMF distribution for a
	$p$-dimensional unit vector $m$ is defined as: 
	$$f_p(m; \mu, \kappa) = C_p(\kappa) \exp(\kappa \mu^T m),$$ 
	where $||\mu|| = 1$, $\kappa \geq 0$, $C_p(\kappa) =
	\frac{\kappa^{p/2-1}}{(2\pi)^{p/2} I_{p/2-1}(\kappa)}$, and $p$ is the
	dimension of the vector. $I_v(\cdot)$ is the modified Bessel function of the
	first kind at order $v$ and is defined as $I_v(\kappa) = \sum_{q=0}^{\infty}
	\frac{1}{q!  \Gamma(q+v+1)} (\frac{\kappa}{2})^{2q+v}$, and $\Gamma(\cdot)$ is
	the gamma function. 
	
	
	\subsection{Parameter Inference}
	
	In our \our model, the parameters to be estimated are the parameters $(\pi, A)$
	for the hidden states and the distribution parameters $(\mu_t, \sigma_t, \mu_l,
	\Sigma_l, \mu, \kappa)$. Since we are using standard Gaussian distributions for
	modeling location and time, the updating rules for all the parameters except
	$\kappa$ can be easily derived by the Baum-Welch algorithm
	\cite{baum1970maximization} --- an Expectation-Maximization procedure for HMM.
	The challenge of applying the Baum-Welch algorithm is how to estimate the
	parameter $\kappa$.
	
	Due to the complicated form of the vMF distribution, it is intractable to
	derive closed-form solutions for $\kappa$ in the M-step of the Baum-Welch
	algorithm.  However, we found that one can use Newton's method to find an
	approximate solution of $\kappa$ that asymptotically converges to the optimal
	value. Below we first present our method for updating $\kappa$ based on Banerjee's work \cite{banerjee2005clustering} and 
	Newton's method, and then show our theoretical
	analysis that our update rule achieves quadratic convergence rate.
	In the M-step of the Baum-Welch algorithm, we estimate the $\kappa$ value with
	the following iterative procedure:
	\begin{align*}
	&\qquad \bar{r} \leftarrow \frac{||\sum_{i=1}^{N} m_i||}{N}\\
	&\qquad\kappa \leftarrow \frac{\bar{r}p-\bar{r}^3}{1-\bar{r}^2}\\
	&\textbf{repeat}\\
	&\qquad\kappa \leftarrow \kappa - \frac{A_p(\kappa) - \bar{r}}{1-A_p^2(\kappa) - \frac{p-1}{\kappa}A_p(\kappa)}\\
	&\textbf{until} \;	\text{convergence}
	\end{align*}
	where $m_i$ is a $l_2$-normalized $p$-dimensional text embedding, $N$ is the
	total number of text embeddings belonging to the current state, and $A_p(\kappa) =
	\frac{I_{p/2}(\kappa)}{I_{p/2-1}(\kappa)}$. 

	\section{Theoretical Analysis}
	
	
	Due to the complicated mathematical form of the vMF distribution, no existing
	literature has proved that the vMF distribution can work under the EM
	framework.  In this section, we theoretically prove that:
	
	\begin{enumerate}
		\item The EM algorithm is able to work with the vMF distribution, because
		there exists a unique $\kappa$ such that the Q-function in the M-step can
		be maximized.
		\item While closed-form solutions for $\kappa$ are hard to be obtained, one
		can use Newton's method for obtaining an approximate solution, which is
		guaranteed to converge to the optimal $\kappa$ for M-step with quadratic
		convergence rate.
	\end{enumerate}
	
	\begin{theorem}
		There exists a unique $\kappa$ that maximizes the Q-function in the M-step of the EM algorithm.  
	\end{theorem}
	\begin{proof}
		To maximize the Q-function, it is equivalent to solve $A_p(\kappa) = \bar{r}$
		where $\bar{r} = \frac{||\sum_{i=1}^{N} m_i
			||}{N}$ \cite{banerjee2005clustering}. Based on this result, we have the
		following claims:
		\begin{enumerate}
			\item Claim 1: $0< \bar{r} \leq 1$.
			
			\textit{Proof.} It is obvious that $\bar{r} > 0$. Since $m_{i1}^2+m_{i2}^2 + \ldots + m_{ip}^2 = 1$, we have $||\sum_{i=1}^{N} m_i||^2 = (m_{11}+...+m_{N1})^2 + ... + (m_{1p}+...+m_{Np})^2  \leq N(m_{11}^2+ ... + m_{N1}^2 +  ... + m_{1p}^2+ ... + m_{Np}^2)  = N^2$. Hence, $\bar{r} = \frac{||\sum_{i=1}^{N} m_i ||}{N} \leq 1$. 
			
			\item Claim 2: $\lim\limits_{\kappa \rightarrow 0} A_p(\kappa) = 0$, $\lim\limits_{\kappa \rightarrow \infty} A_p(\kappa) = 1$
			
			\textit{Proof.} The first equation is given by Lemma 2.1 in \cite{segura2011bounds}. With Corrolary 1 in \cite{balachandran2013exponential}, we have $ \exp(-\frac{p}{2\kappa}) \leq A_p(\kappa) \leq \exp(-\alpha_0\frac{p-1}{2\kappa})$, where $\alpha_0 = -log(\sqrt{2}-1)$, if $p \leq 2\kappa$. 
			Hence, we have $ \lim\limits_{\kappa \rightarrow \infty}\exp(-\frac{p}{2\kappa}) \leq \lim\limits_{\kappa \rightarrow \infty}A_p(\kappa) \leq \lim\limits_{\kappa \rightarrow \infty}\exp(-\alpha_0\frac{p-1}{2\kappa})$. Therefore, $\lim\limits_{\kappa \rightarrow \infty}A_p(\kappa) = 1$. 
			
			\item Claim 3: $A_p(\kappa)$ is a continuous function if $\kappa$ is real-valued and positive. 
			
			\textit{Proof.} By the definition of the modified Bessel function and its recurrence relation (Equation 9.6.1 and 9.6.26 in \cite{abramowitz1964handbook}), we can get $A_p'(\kappa) = 1-A_p^2(\kappa) - \frac{p-1}{\kappa}A_p(\kappa)$. Since $A_p(\kappa)$ is differentiable when $\kappa$ is real-valued and positive, $A_p(\kappa)$ is continuous at all real-valued positive $\kappa$. 
		\end{enumerate}
		By the intermediate value theorem and the above claims, we have that
		there exists a solution for $A_p(\kappa) = \bar{r}$. Since we have
		$A_p'(\kappa) > 0$ for all positive $\kappa$ (Equation 15 in
		\cite{amos1974computation}), there exists a unique $\kappa$ such that
		$A_p(\kappa) = \bar{r}$ and the Q-function is maximized. This completes the
		proof. 
	\end{proof}
	
	So far, we have proved that there exists one unique $\kappa$ to maximize the
	Q-function. This implies the EM algorithm is able to work with vMF distributions.
	However, it is non-trivial to find the optimal $\kappa$ since it involves
	calculating the ratios of modified Bessel functions. Next, we show that the
	solution can be found by Newton's Method.  
	
	\begin{lemma} 
		$A_p(\kappa)$ is a strictly concave function, when $\kappa > 0$ and $p \geq 2$.
	\end{lemma}
	\begin{proof} 
		It is shown in Theorem 3 in \cite{simpson1984some} that $
		\frac{\kappa}{A_p(\kappa)}$ is strictly convex for $\kappa \geq 0$, $p \geq
		2$. Thus, $\frac{-\kappa A_p''(\kappa) A_p^2(\kappa) - 2A_p(\kappa)
			A_p'(\kappa) (A_p(\kappa) - \kappa A_p'(\kappa))}{A_p^4(\kappa)} > 0 $.
		Since $A_p(\kappa) > 0$, we get
		$A_p''(\kappa) < \frac{2A_p'(\kappa) }{\kappa A_p(\kappa)} (\kappa A_p'(\kappa) - A_p(\kappa))$. 
		Now we have: \textbf{1)} $A_p'(\kappa) >0$ and $\kappa A_p'(\kappa) -
		A_p(\kappa) < 0$, (Equation 15 in \cite{amos1974computation}); \textbf{2)}
		$A_p(\kappa) = \frac{I_{p/2}(\kappa)}{I_{p/2-1}(\kappa)} > 0 $, since
		$I_{v}(\kappa)$ is positive when $v \geq 0$, and $\kappa > 0$ (Equation 9.6.1
		in \cite{abramowitz1964handbook}). Therefore,
		$A_p''(\kappa) < 0$ and it is strictly concave. This completes the proof of
		Lemma 1. 
	\end{proof}
	
	Building on Lemma 1, we proceed to show that Newton's Method is guaranteed to
	converge to the solution for $A_p(\kappa) = \bar{r}$.
	
	\begin{theorem} Newton's method is guaranteed to converge to the solution for
		$A_p(\kappa) = \bar{r}$. 
	\end{theorem}
	\begin{proof} 
		Assume $\kappa = r$ is the solution, \ie $A_p(r)-\bar{r}=0$. Let's start with a point $\kappa_0$ and $0<\kappa_0 \leq r$. We define $e_n = \kappa_n -r$. By Newton's updating rule, we have
		\begin{equation}
		e_{n+1} = e_n - \frac{A_p(\kappa_n)-\bar{r}}{A_p'(\kappa_n)}
		\end{equation}
		By Taylor's Theorem, we have $A_p(r) = A_p(\kappa_n-e_n) =  A_p(\kappa_n)-e_nA_p'(\kappa_n)+\frac{1}{2}e_n^2A_p''(\xi_n) = \bar{r} $, where $\kappa_n \leq \xi_n \leq r$. 
		Then we have
		\begin{equation}
		e_{n+1} =  \frac{e_nA_p'(\kappa_n) - A_p(\kappa_n)+\bar{r}}{A_p'(\kappa_n)} = \frac{1}{2} \frac{A_p''(\xi_n)}{A_p'(\kappa_n)} e_n^2.
		\end{equation}
		Therefore, $e_n \leq 0$ for all $n$ since $A_p''(\xi_n) < 0 $ and
		$A_p'(\kappa_n) > 0$. That implies $\kappa_n \leq r$ and $A_p(\kappa_n) \leq
		\bar{r}$ for all $n$. Therefore, from Equation (1), we have $e_{n+1} \geq
		e_n$. Thus, the sequence $e_0, e_1, \ldots, e_n$ is an increasing sequence and
		bounded by 0, and therefore, it must have a limit $e^*$. Accordingly, the sequence
		$\kappa_0, \kappa_1, \ldots, \kappa_n$ must  have a limit $\kappa^*$. From
		Equation (1), we have $\lim\limits_{n \rightarrow \infty} e_{n+1} =
		\lim\limits_{n \rightarrow \infty} (e_n -
		\frac{A_p(\kappa_n)-\bar{r}}{A_p'(\kappa_n)})$, and thus $A_p(\kappa^*) =
		\bar{r}$ and $\kappa^* = r$. Therefore, if we choose any positive starting
		point $\kappa_0$, such that $\kappa_0 \leq r$, Newton's method is guaranteed to
		converge to the solution.
	\end{proof}
	
	We have shown that Newton's method is guaranteed to converge to the solution
	for $A_p(\kappa) = \bar{r}$. The update rule is simply given by $\kappa_{n+1} =
	\kappa_n - \frac{A_p(\kappa_n) - \bar{r}}{A_p'(\kappa_n)}$, where
	$A_p'(\kappa_n) = 1-A_p^2(\kappa_n) - \frac{p-1}{\kappa_n}A_p(\kappa_n)$. Next,
	we show the convergence rate of Newton's method.
	
	\begin{lemma} 
		The rate of convergence for calculating $A_p(\kappa) = \bar{r}$ using Newton's
		Method is quadratic and $e_{n+1} \approx Ce_n^2$ with $C \in (-1, 0)$. 
	\end{lemma} 
	\begin{proof} Equation (2) shows the convergence rate is quadratic. Then, we have $e_{n+1} = \frac{1}{2} \frac{A_p''(\xi_n)}{A_p'(\kappa_n)} e_n^2 \approx \frac{1}{2} \frac{A_p''(\kappa_n)}{A_p'(\kappa_n)} e_n^2$ and  
		\begin{equation*}
		\begin{split}
		\frac{A_p''(\kappa_n)}{A_p'(\kappa_n)} 
		&=	\frac{(-2A_p(\kappa_n) - \frac{p-1}{\kappa_n} )A_p'(\kappa_n) + \frac{p-1}{\kappa_n^2}A_p(\kappa_n) } {A_p'(\kappa_n)} \\
		&> -2A_p(\kappa_n) - \frac{p-1}{\kappa_n} + \frac{p-1}{\kappa_n^2}\kappa_n = -2A_p(\kappa_n)
		\end{split}
		\end{equation*}
		where the last inequality uses Equation 15 in \cite{amos1974computation}. Therefore, $\frac{A_p''(\kappa_n)}{A_p'(\kappa_n)} \in (-2A_p(\kappa_n), 0)$. We have shown that $\lim\limits_{\kappa \rightarrow 0} A_p(\kappa) = 0$, $\lim\limits_{\kappa \rightarrow \infty} A_p(\kappa) = 1$ and $A_p'(\kappa) > 0$ in Theorem 1, and hence $A_p(\kappa_n) \in (0,1)$. Therefore, $e_{n+1} \approx Ce_n^2$, where $C \in (-1, 0)$. 
	\end{proof}
	

	\section{Experiments}
	
	In this section, we empirically evaluate the performance of \our.  We
	implemented \our and the baseline methods in JAVA and conducted all the
	experiments on a computer with 2.9 GHz Intel Core i7 CPU and 16GB memory.
	
	\subsection{Experimental Setup}
	
	\subsubsection{Data}
	
	We use both synthetic and real-life datasets to evaluate \our.  The real-life
	datasets are semantic traces from Twitter users collected by Zhang
	\etal \cite{zhang2016gmove}.  The first dataset (LA) consists of million-scale
	geo-tagged tweets created by Los Angeles users from 2014.08.01 to 2014.11.30.
	Following the preprocessing steps in \cite{zhang2016gmove}, we first group the
	tweets by user ID to obtain the location history for each user.  Since two
	consecutive records in a raw location history can be large, we further segment
	the location history into dense semantic traces with a time threshold $\Delta t
	= 6 ~h$, such that the time gap between any two consecutive records is no
	larger than 6 hours. After preprocessing, we obtain approximately 30 thousand
	semantic traces.  The second dataset (NY) consists of the geo-tagged tweets in
	New York City during 2014.08.01 and 2014.11.30. We preprocess the NY data in a
	similar manner and obtain approximately 42 thousand trajectories in total.
	
	We also generate multiple synthetic data to verify the theoretical analysis of
	our \our model. For generating points from the vMF distribution, we use the
	code from Chen \textit{et al.} \cite{chen2015parameter}. Given the synthetic
	points, we apply our proposed Newton's method for estimating the parameters of
	the \our model, and evaluate the approximation errors and convergence speed.

	

	\subsubsection{Baseline Methods}
	
	We compare our \our model with the following baseline methods: 
	
	\begin{enumerate}
		
		\item \textsc{Law} \cite{brockmann2006scaling} is a widely used mobility
		model based on the L\'{e}vy flight law with long-tailed distributions. 
		
		\item \textsc{HMM} \cite{MathewRM12} uses HMM to model the spatial locations 
		in the trace data for mobility modeling.
		
		\item \textsc{ST-HMM} is an extension of \textsc{HMM} that models both
		spatial and temporal information in the trace data.
		
		\item GMove \cite{zhang2016gmove} is the state-of-the-art mobility
		model for semantic traces. It differs from \our in the text modeling part.
		It uses the bag-of-words model to represent text messages and uses
		multinomial distributions to generate the observed messages.
		
		\item \textsc{GHMM} is an adaption of \our. It uses independent Gaussians to model text vectors instead of the vMF distribution. 
		
	\end{enumerate}
	
	
	\subsubsection{Evaluation Protocol}
	
	
	We use next location prediction as a downstream task for evaluating the quality
	of the \our model.  Given a semantic trace dataset, we randomly
	select 70\% traces for model training and use the rest 30\% for testing.  For a
	test trajectory $[x_1, x_2, ..., x_R]$, we assume the first $R-1$ locations
	$[x_1, x_2, ..., x_{R-1}]$ are observed and attempt to recover the last record
	$x_R$.  Specifically, we first form a candidate pool by mixing $x_R$ with other
	records whose creating time and distances are close
	to $x_R$.  \footnote{To select negative records, we set the distance threshold
		to 3.5 kilometers on LA and 2.0 kilometers on NY; and we set the time
		threshold to 300 seconds on both LA and NY.} After the candidate pool is
	formed, we use the \our model to select the top-$K$ most
	likely visited records and see whether the ground-truth appears in the top-$K$
	list. We use the prediction accuracy $@K$ to measure the performance of
	different models, \ie, the percentage of test traces for which the ground-truth
	record is recovered by the top-$K$ list.
	

	

	\subsection{Results on Synthetic Data}
	
	Figure \ref{fig:synthetic1} shows the performance of our used Newton's method
	for estimating the parameters of a vMF distribution on synthetic data. As shown
	in Figure \ref{fig:synthetic1} (a), our estimation method converges extremely fast for estimating
	$\kappa$, achieving approximation errors smaller than $10^{-13}$
	after three iterations.  Figure \ref{fig:synthetic1} (b) shows the $\mu$ and $\kappa$ estimation performance on
	synthetic datasets with different sizes.  We can see the approximation error
	tends to become smaller on synthetic data sets with more samples.  This is
	expected, as a small number of samples may lead to biased estimations of the true
	parameter values. The results in Figure \ref{fig:synthetic2} shows the estimation
	performance for different $\kappa$ and dimension  $p$ (the number of samples is
	100,000). Generally, we do not observe obvious patterns showing how the
	approximation errors change with different $\kappa$ and $p$, but the relative
	approximation errors are quite small under different $\kappa$ and  $p$ values.
	
	\begin{figure}[ht!]
		\centering
		\subfigure[$\kappa$ convergence plot]{\includegraphics[width=0.49\linewidth]{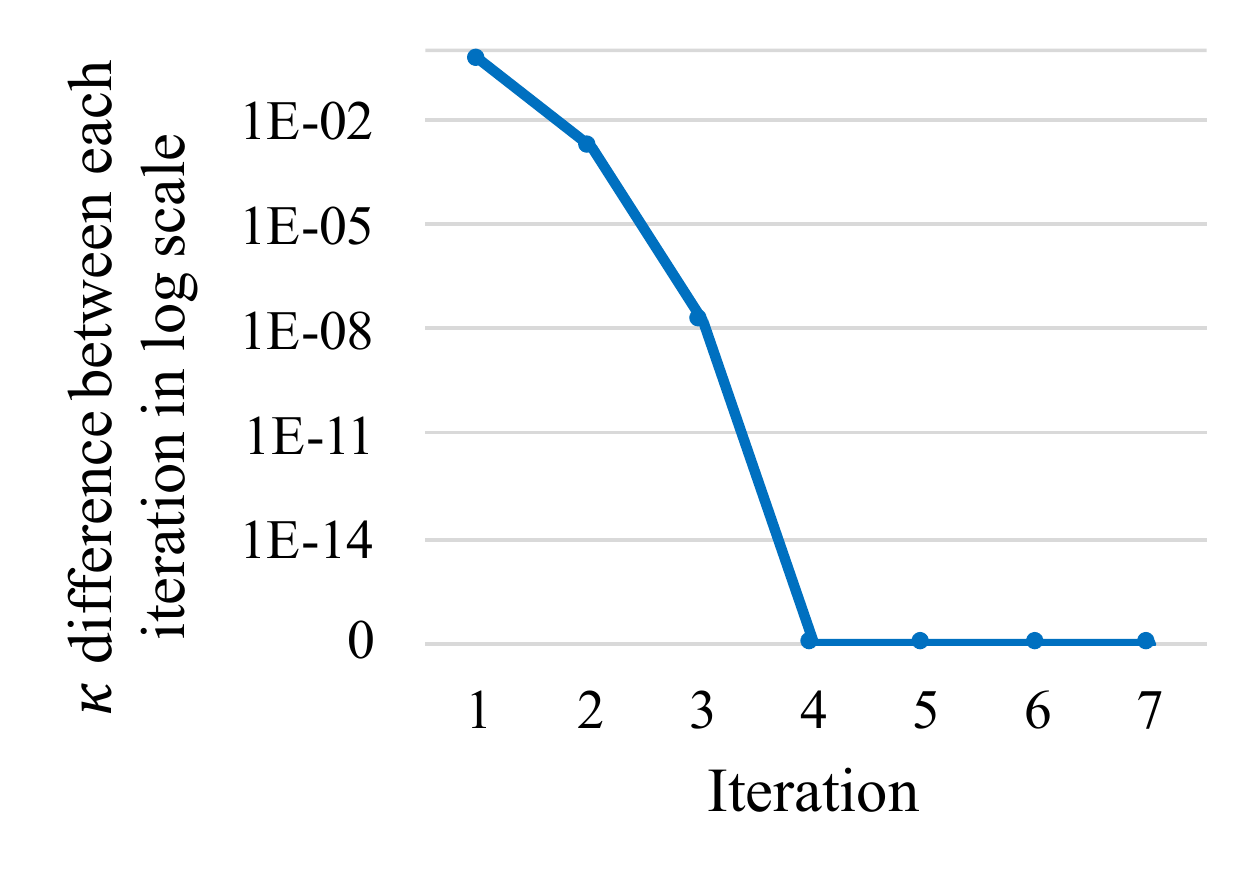}}
		\subfigure[vMF parameter estimation error with $p= 100$ and $\kappa=100$]
		{\includegraphics[width=0.49\linewidth]{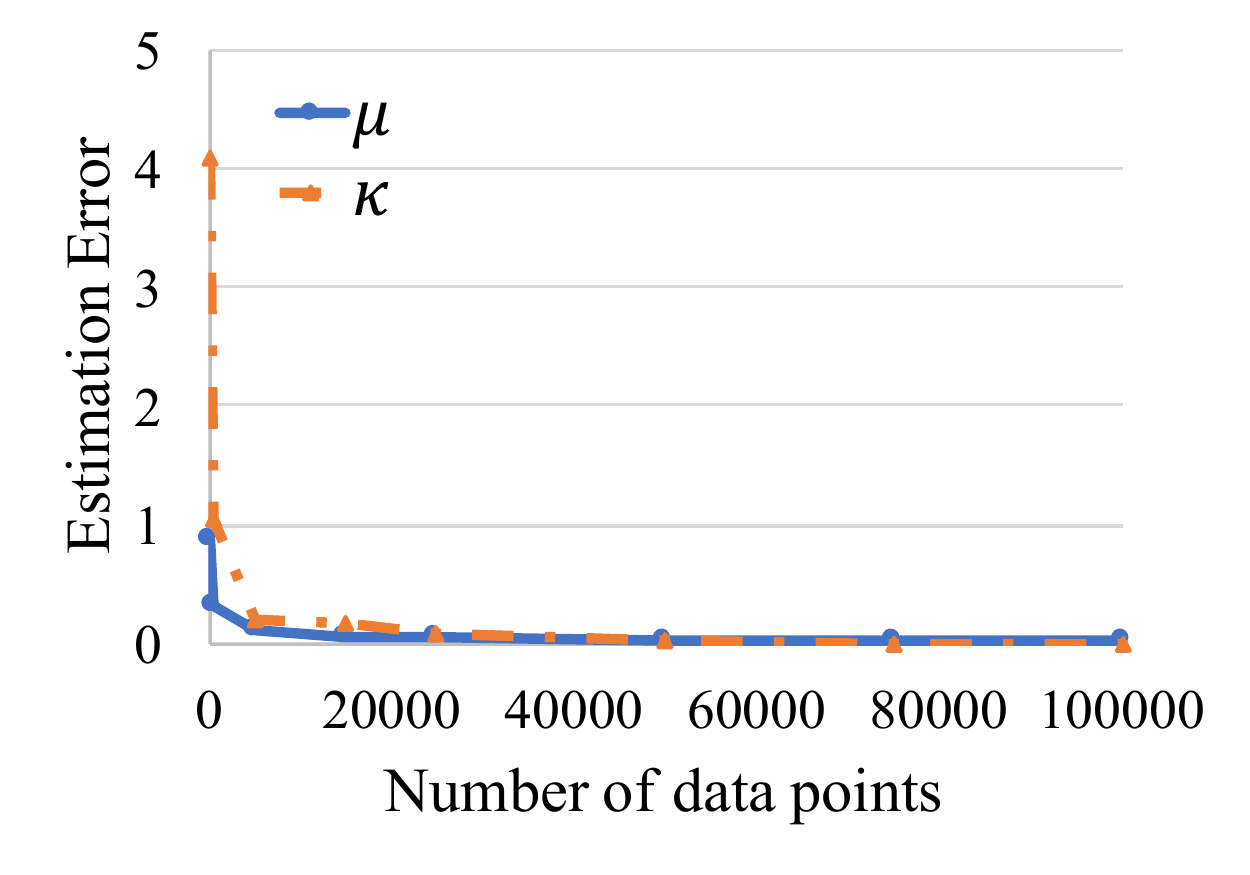}}
		\caption{The convergence of the Newton's method.}
		\label{fig:synthetic1}
	\end{figure}
	
	\begin{figure}[ht!]
		\centering
		\subfigure[varying $\kappa$ with $p= 100$]{\includegraphics[width=0.49\linewidth]{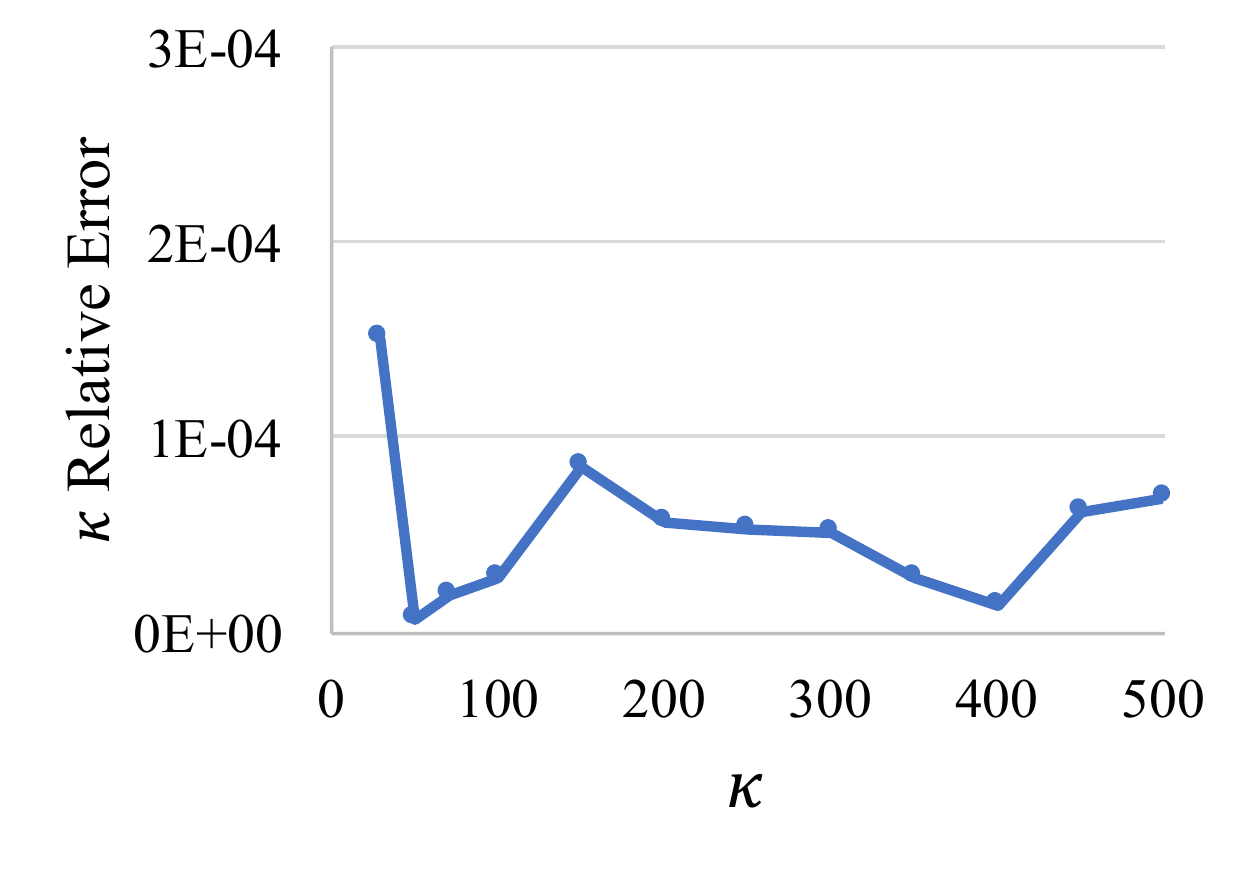}}
		\subfigure[varying $p$ with $\kappa = 100$]{\includegraphics[width=0.49\linewidth]{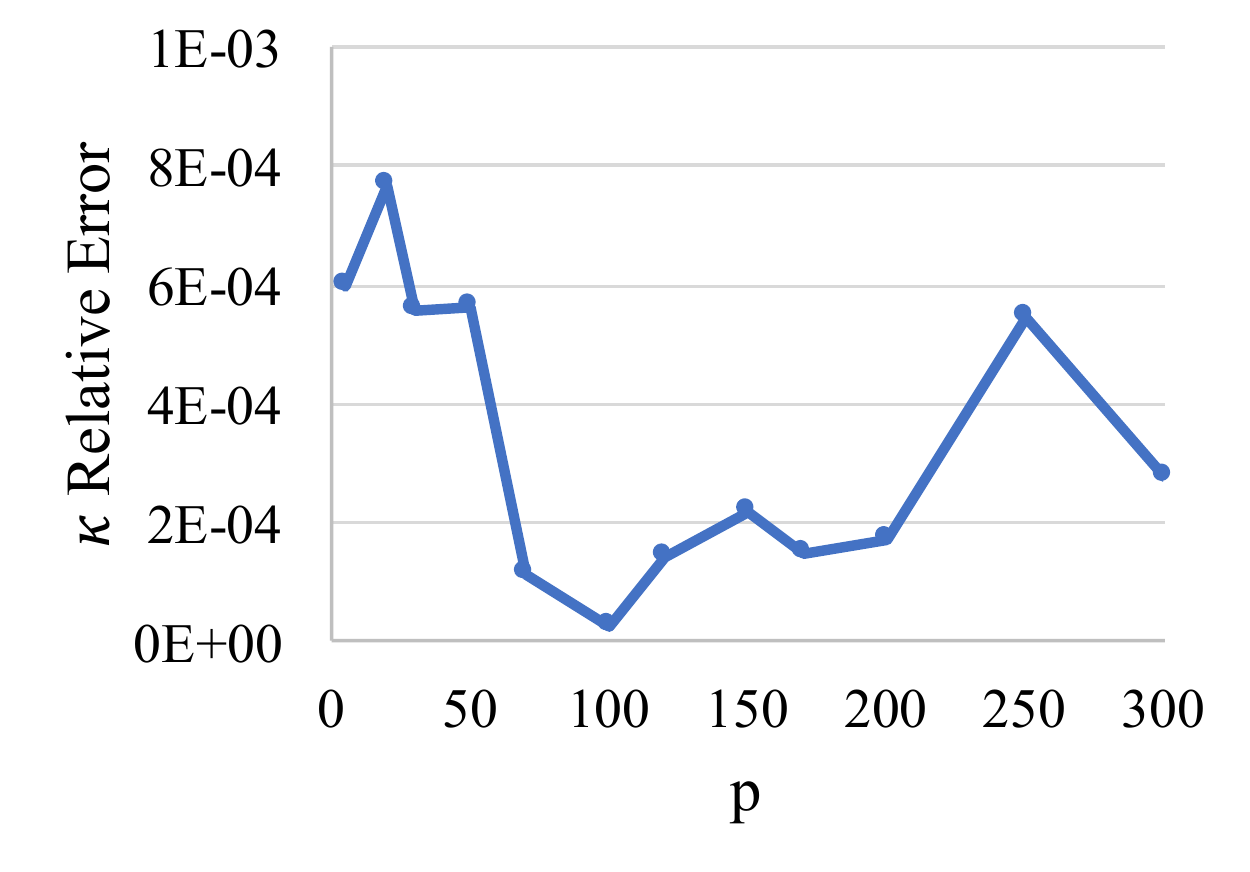}}
		\caption{Estimation error for the concentration parameter $\kappa$}
		\label{fig:synthetic2}
	\end{figure}
	
	
	
	
	\subsection{Results on Real-Life Data}

	\subsubsection{Visualization of the Mobility Models}
	
	In this set of experiments, we set the number of states to 50 on LA and NY to
	obtain mobility models. After parameter inference, each state is characterized by: (1) a
	two-dimensional Gaussian distribution for the spatial location; (2) a
	one-dimensional Gaussian distribution for the time; and (3) a 30-dimensional
	vMF distribution for the semantics. 
	
	\begin{figure*}[ht!]
		\centering
		\subfigure[LA dataset state keywords and transitions]{\includegraphics[width=1\linewidth]{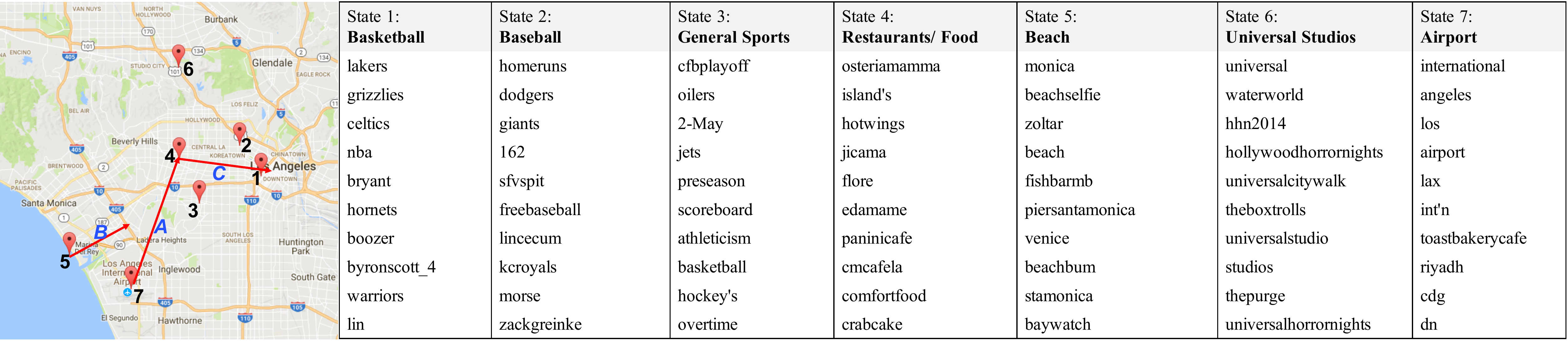}}
		\subfigure[NY dataset state keywords and transitions]{\includegraphics[width=1\linewidth]{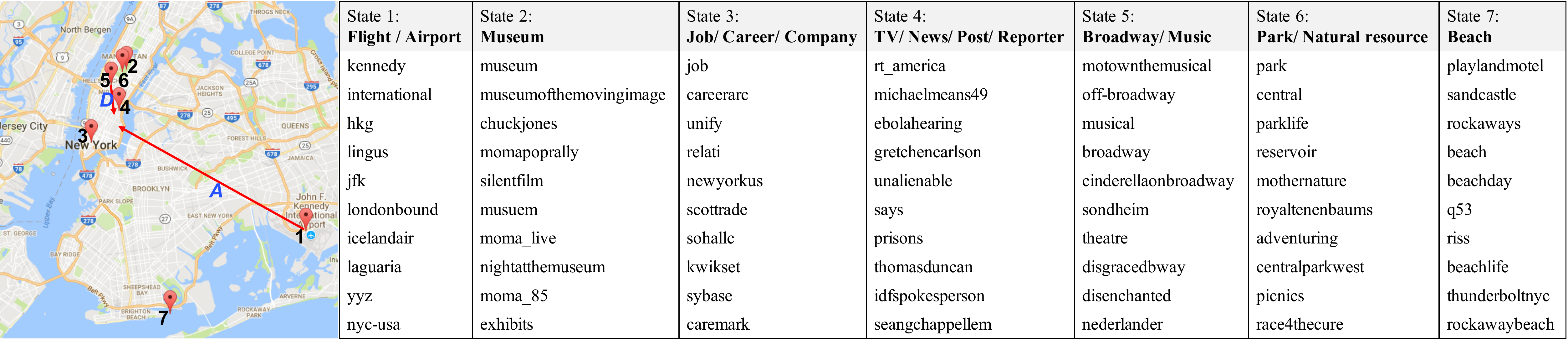}}
		\caption{The visualization of the mobility models. The digits (1-7) indicate
			different state locations and the letters (A-D) indicate frequent human
			movements.
		}
		\label{fig:lastatetopics}
	\end{figure*}
	
	
	Figure \ref{fig:lastatetopics}
	visualizes a number of representative states and some frequent transitions
	among them.  We plot
	the mean location of some states as well as the top-10 keywords from the
	vocabulary whose embeddings are the closest to the vMF mean directions.
	Most of the top-10 keywords for the same state carry consistent and clear
	semantics. For example, for the \textsc{Baseball} state on the LA dataset:
	\textit{homeruns} is a specific baseball term; \textit{dodgers} is the baseball
	team in LA; \textit{giants} is the baseball team in San Francisco; \textit{162}
	indicates that there are 162 games for each team in the Major League Baseball
	(MLB) season; and the rest six keywords are all related to baseball too. We
	have examined the center locations of the states, and found that the
	geographical locations well match the semantic meanings of different states.
	
	Another interesting finding is that in LA dataset, the mean direction of the
	\textit{General Sports} state lies in-between the \textit{Basketball} state and
	the \textit{Baseball} state in the embedding space. Also, the concentration
	parameter $\kappa$ for \textit{General Sports} state is lower than
	\textit{Basketball} and \textit{Baseball} state. Such phenomenon intuitively
	makes sense since \textit{General Sports} is a  broader topic and the
	semantics of the tweets are more scattered.


	We have also observed some interesting state transitions. As shown in Figure
	\ref{fig:lastatetopics}, the following transitions receive high probabilities in the \our model:
	(A) moving from airports to restaurants; (B) enjoying beach activities at the
	Venice beach, and then moving around for other leisure activities; (C)
	going to concerts after having food; (D) watching shows at Broadway and then
	having other sightseeing activities in NYC Downtown. These high-probability
	transitions match people's movements in the real world well.


	
	\subsubsection{Performance for Next Location Prediction}

	
	Figure \ref{fig:laresults} shows the performance of next location prediction
	for different mobility models.  It can be seen that our \our model outperforms
	the state-of-the-art GMove model by 3.2\% on average. 
	The performance difference shows that the text embedding can better
	capture the semantics of text messages and reduce text sparsity. Also, the vMF
	distribution 
	unleashes the discriminative power of text embeddings in a directional metric
	space.
	
	\begin{figure}[ht!]
		\centering
		\subfigure[LA]{\includegraphics[width=0.49\linewidth]{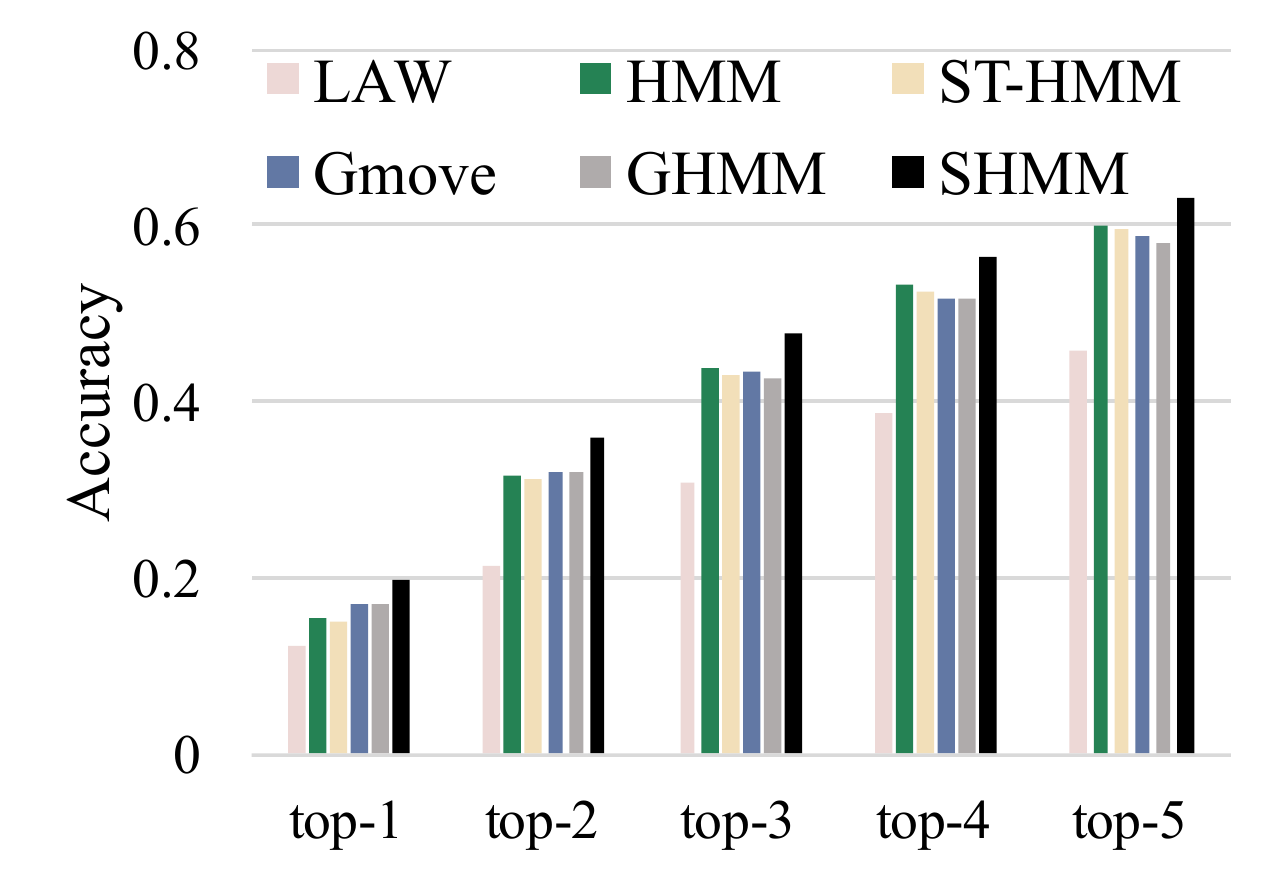}}
		\subfigure[NY]{\includegraphics[width=0.49\linewidth]{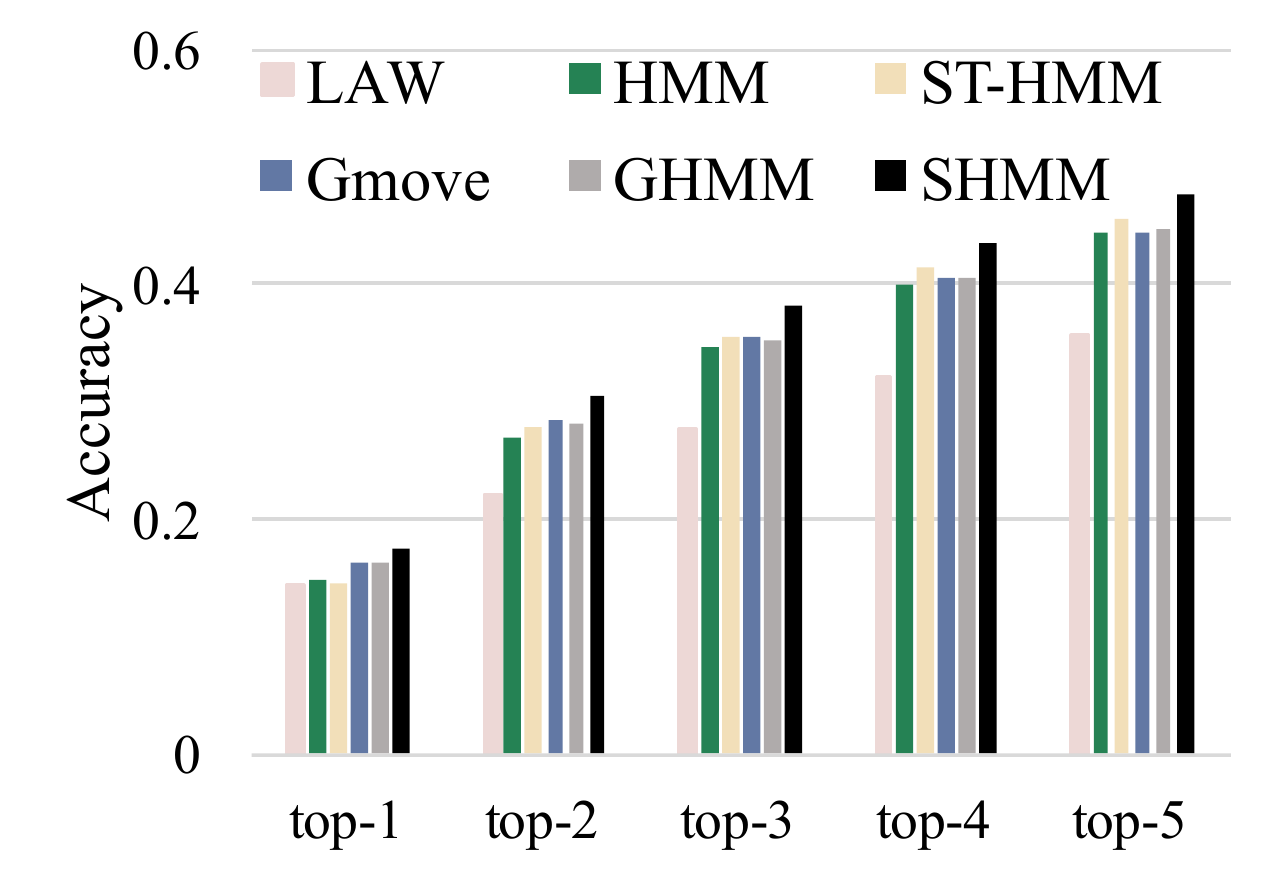}}
		\caption{The accuracies of top-$K$ location prediction.}
		\label{fig:laresults}
	\end{figure}

	\subsubsection{Effects of Parameters}
	
	In Figure \ref{fig:lastates}, we study the performance of \our and GMove when the number
	of states varies.  We find that the performance of both models generally increases
	with the number of states.  One major reason is that the semantics of people's
	activities are separated at finer granularities when the number of states is
	large.  For example, we can see from the LA data set that \textit{General
		Sports}, \textit{Basketball} and \textit{Baseball} are three separated
	topics. 
	If the number of states is not large enough,
	these states may be clustered as one single topic. On the other hand, one
	caveat for choosing the number of states is that a large number of states could
	incur high computational overhead, and also harm the interpretability of the
	result model because the same semantics may be split into duplicate ones. 
	
	\begin{figure}[!ht]
		\centering
		\subfigure[LA]{\includegraphics[width=0.49\linewidth]{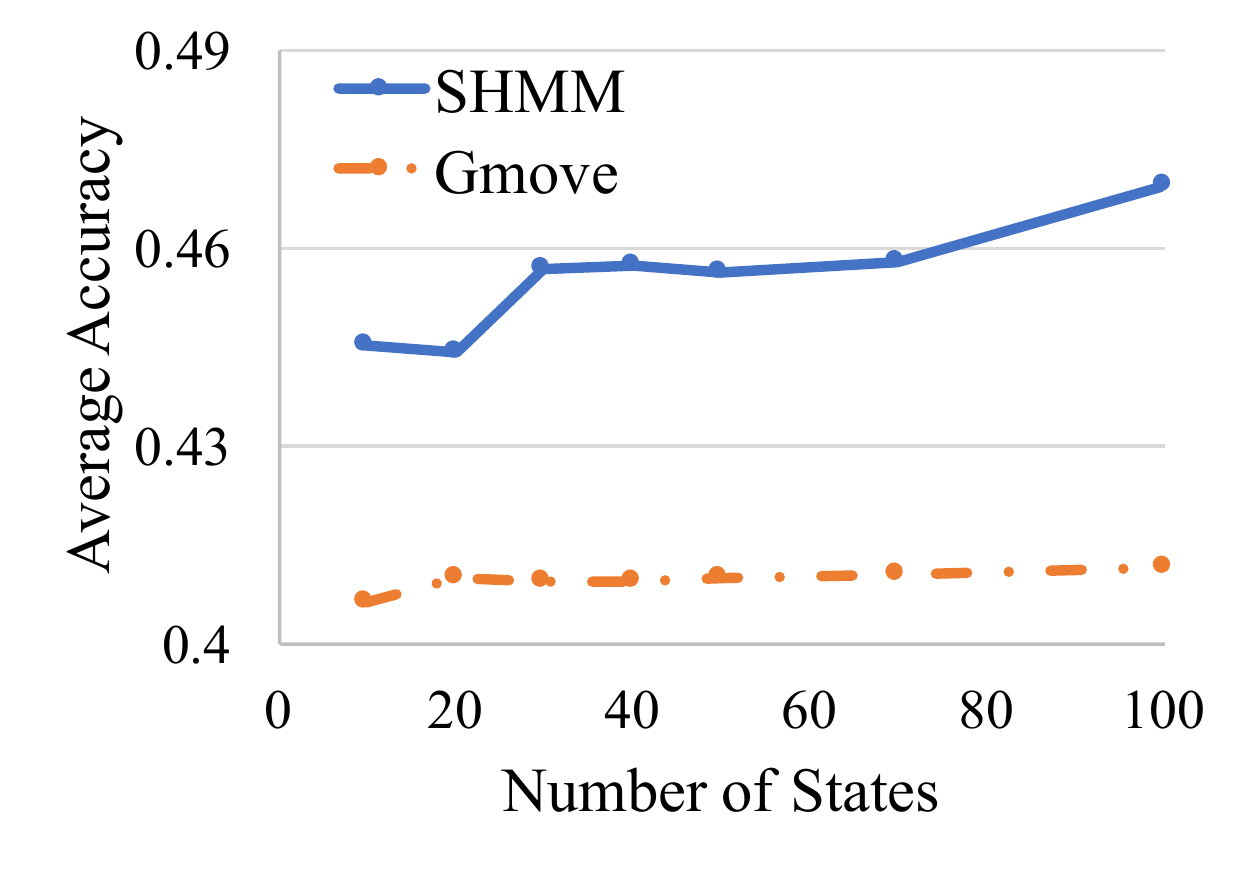}}
		\subfigure[NY]{\includegraphics[width=0.49\linewidth]{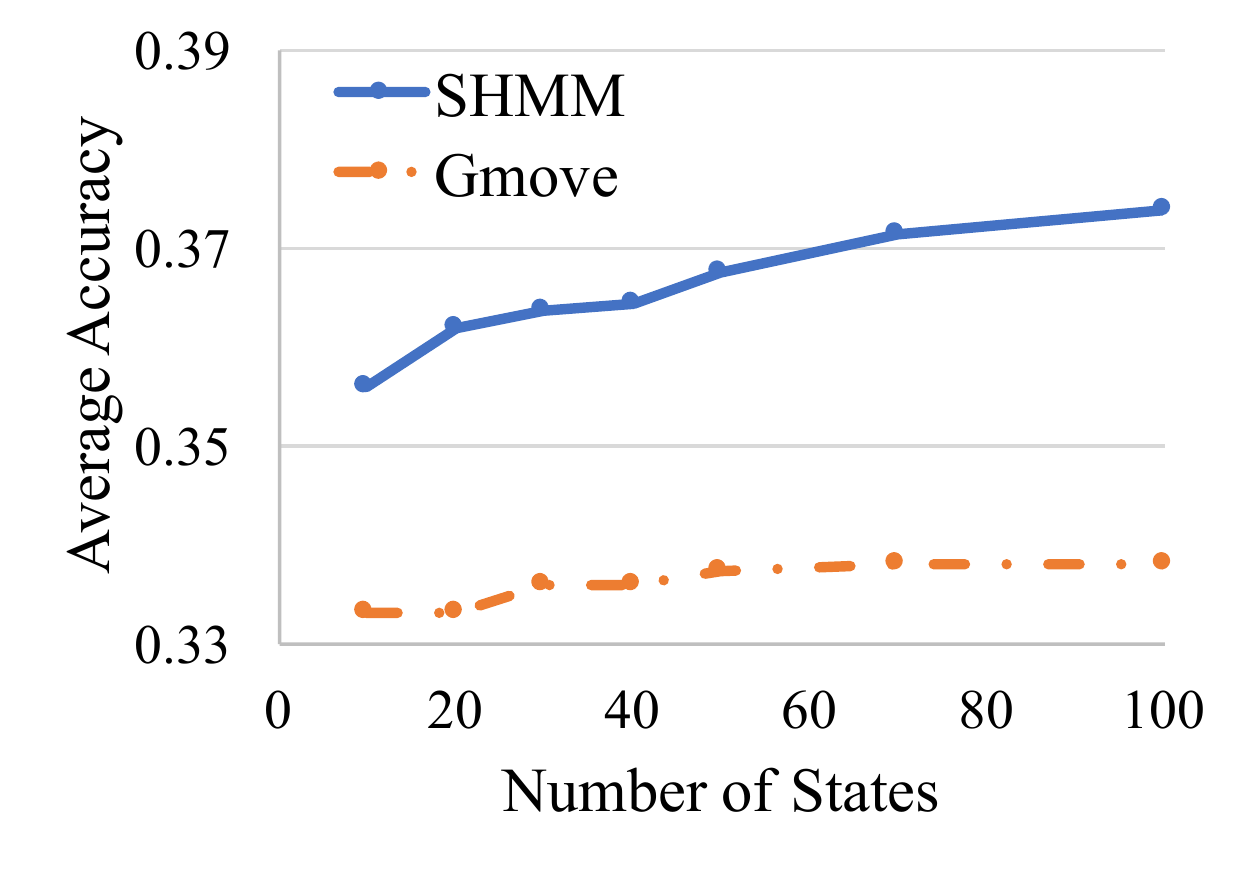}}
		\caption{Prediction accuracy v.s. the number of states.}
		\label{fig:lastates}
	\end{figure}


	\subsubsection{Efficiency Comparison}
	Finally, we compare the modeling training time between \our and GMove when the
	number of states varies. Generally speaking, the training time of both models
	increases quadratically with the number of states.  The training time of \our
	is obviously smaller than that of GMove, \eg, when the number of states is 100,
	training \our is 25.7\% faster on LA and 40.9\% faster on NY, and the speedups
	are even larger when the number of states or the data size increases. This is because \our
	models low-dimensional text embeddings  instead of high-dimensional
	bag-of-words, and thus involves much fewer parameters.  In addition, the
	estimation of the parameters for the vMF distribution is cheap and converges
	fast.
	
	
	
	\begin{figure}[h]
		\centering
		\subfigure[LA]{\includegraphics[width=0.49\linewidth]{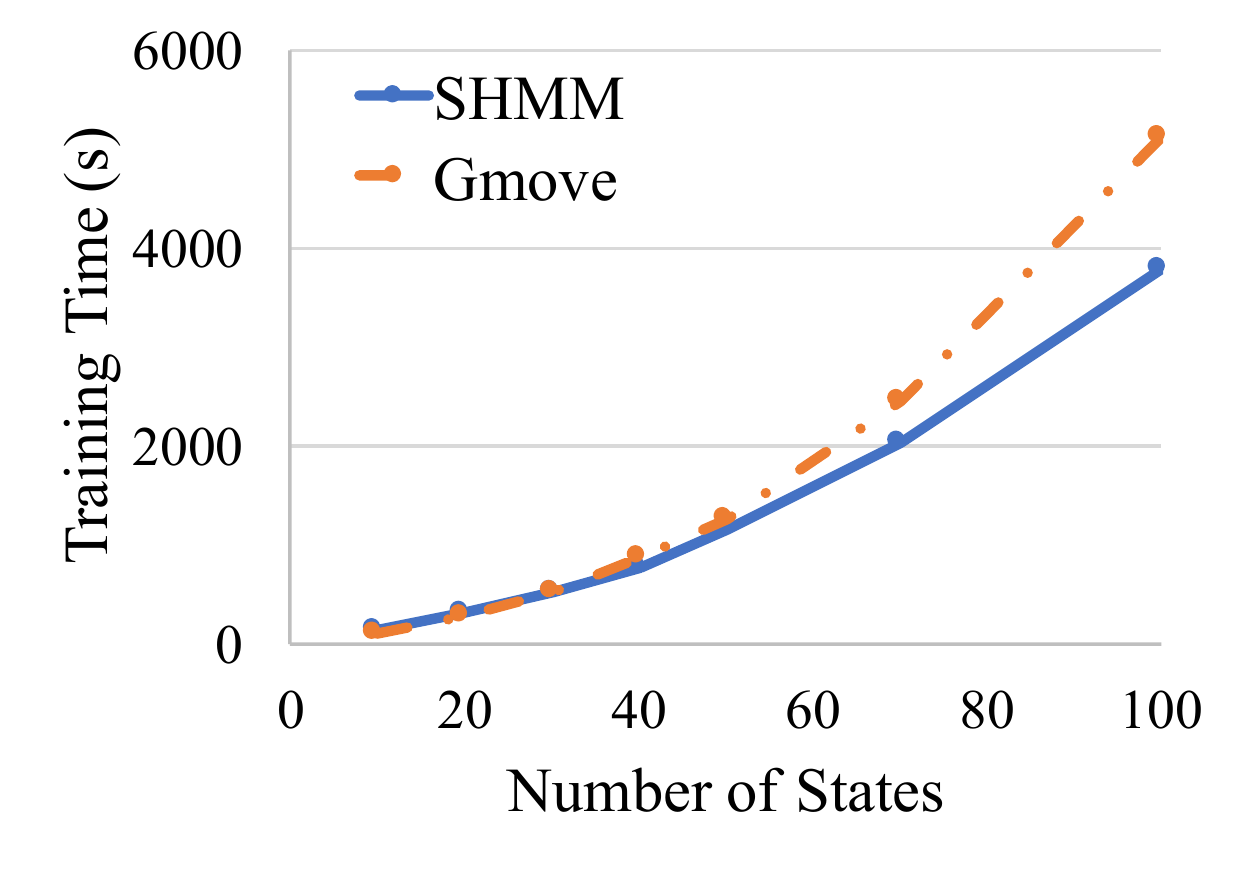}}
		\subfigure[NY]{\includegraphics[width=0.49\linewidth]{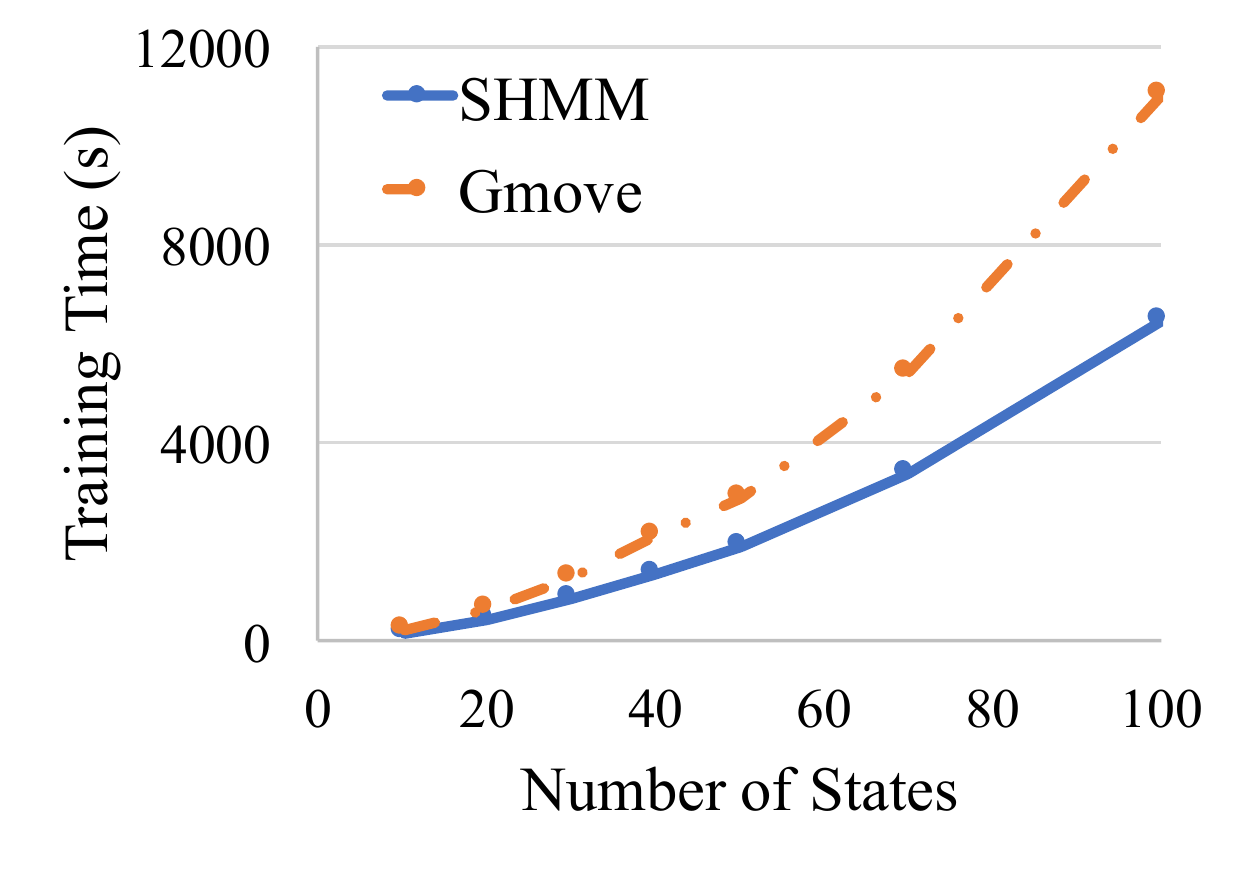}}
		\caption{Running time v.s. the number of states.}
		\label{fig:latrainingtime}
	\end{figure}

	\section{Related Work}
	
	
	\textbf{Human mobility modeling.} Classic human mobility modeling methods focus
	on mining the spatiotemporal regularities underlying human movements.
	Generally, existing mobility modeling methods can be divided into two
	categories: pattern-based methods and
	model-based methods.  Pattern-based
	methods aims at discovering specific mobility patterns that occur regularly.
	Different mobility patterns have been introduced to capture people's movement
	regularities, such as frequent sequential patterns
	\cite{giannotti2007trajectory}, periodic patterns \cite{LiDHKN10}, and
	co-location patterns \cite{KalnisMB05}.  Model-based methods use
	statistical models to characterize the human mobility, and learn the parameters
	of the designed model from the observed trace data.  Mathew \emph{et al.}
	\cite{MathewRM12} use the hidden Markov model to capture the sequential
	transition regularities of human mobility; Brockmann \textit{et al.}
	\cite{brockmann2006scaling} proposed that human mobility can be modeled by a
	continuous-time random-walk model with long-tail distribution; Cho \etal
	\cite{Cho2011} introduce periodic mobility models to discover the periodicity
	underlying human movements.
	
	While the above mobility modeling methods focus on spatiotemporal regularities
	without considering text data, recent years are witnessing growing interest in
	modeling human mobility from semantic trace data
	\cite{Ying2011,wu2015semantic,zhang2016gmove,ZhangHSLP14}.  Among these works, the
	state-of-the-art GMove model  \cite{zhang2016gmove} is the most relevant to our
	model.  Both GMove and \our use hidden Markov models to model the generation
	process of the observed semantic trace data. However, \our is different from
	GMove in that it encodes the semantics of user activities with text embeddings,
	and uses the vMF distribution to model the text embeddings in the HMM model. As
	such, the \our involves much fewer parameters and well unleashes the
	discriminative power of text embeddings in a directional metric space.

	It is worth mentioning that, there are quite a number of works that uses human
	trace data for the location prediction problem
	\cite{wang2015regularity,liu2016predicting}.  Typically, they extract features
	that are important for predicting which place the user tends to visit next
	based on discriminative models such as recurrent neural networks.  While we use
	location prediction as an evaluation task in our experiments, our work is quite
	different from these works. Instead of optimizing the performance of location
	prediction, our focus is to learn interpretable  models that reveals the
	regularities underlying human movements. Besides location prediction, our
	learned mobility models can be used for many other downstream tasks as well.

	
	\vpara{vMF-based learning.} There are some existing works that utilize vMF
	distribution for different learning tasks. Dhillon \textit{et al.}
	\cite{dhillon2003modeling} and Banerjee \textit{et al.}
	\cite{banerjee2005clustering} are two pioneering works that use the vMF
	distributions to handle directional data, which demonstrate inspiring results
	for text categorization and gene expression analysis.  
	Besides, Gopal and Yang
	\cite{gopal2014mises} recently applied vMF distributions for clustering
	analysis, and proposed variational inference procedures for estimating the
	parameters of the vMF clustering model.  Batmanghelich \textit{et al.}
	\cite{batmanghelich2016nonparametric} proposed a spherical topic model based on
	the vMF distribution, which accounts for word semantic regularities in language
	and has been demonstrated to be superior than multi-variate Gaussian
	distributions. However, there are no previous works that integrate the vMF
	distribution with HMMs for semantic trace data. To the best of our knowledge,
	we are the first to demonstrate that the vMF distribution can work well with
	HMM for directional data with theoretical guarantees.
	
	
	

	\section{Conclusion}
	
	We proposed a spherical hidden Markov model for learning interpretable human
	mobility model from semantic trace data. Our model uses text embeddings to
	capture the semantics of text messages and integrate the vMF distribution into
	the hidden Markov model for generating such text embeddings.  We have
	theoretically proved that the Expectation-Maximization algorithm is able to
	work with vMF distribution, and that the Newton's method can be applied for
	efficiently solving the M-step with quadratic convergence rate.  Our
	experiments on synthetic data simulations verify our theoretical analysis.
	Furthermore, by applying our model to real-life semantic trace datasets, we are
	able to obtain highly interpretable mobility models, which intuitively make
	sense and outperform baseline models for downstream tasks like location
	prediction.

	\section{Acknowledgments} This work was sponsored in part by the U.S. Army
	Research Lab.  under Cooperative Agreement No.  W911NF-09-2-0053 (NSCTA),
	National Science Foundation IIS-1017362, IIS-1320617, and IIS-1354329,
	HDTRA1-10-1-0120, and Grant 1U54GM114838 awarded by NIGMS.
	
	\bibliography{Bibliography-File}
	\bibliographystyle{aaai}
	
\end{document}